\documentclass[twocolumn,pre,floatfix]{revtex4}

\usepackage[T1]{fontenc}
\usepackage[utf8]{inputenc}
\usepackage{comment}
\usepackage{amsmath}
\usepackage{amsfonts}
\usepackage{amssymb}
\usepackage{mathtools}
\usepackage{graphicx}
\usepackage{booktabs}
\graphicspath{{images/}}
\usepackage{amsthm}
\newtheorem*{theorem}{Theorem}
\newtheorem*{lemma}{Lemma}

\usepackage[pdfstartview=FitH,
            breaklinks=true,
            bookmarksopen=false,
            bookmarksnumbered=true,
            colorlinks=true,
            linkcolor=black,
            citecolor=black,
            urlcolor=black,
            pdftitle={ClusterEval},
            pdfauthor={Andreas Tiffeau-Mayer},
            pdfsubject={Clustering Evaluation by the Homogeneity--Parsimony Trade-Off},
            ]{hyperref}

\begin{document}

\title{External Clustering Validation by the Homogeneity--Parsimony Trade-Off}
\author{Andreas Tiffeau-Mayer}
\affiliation{Division of Infection and Immunity \& Institute for the Physics of Living Systems, University College London}
\date{\today}

\begin{abstract}
Scalar metrics are often used to evaluate clusterings against known classes, but they can obscure a fundamental trade-off: clusterings should be informative about class labels while avoiding unnecessary fragmentation. Here we describe normalized scores of cluster homogeneity and parsimony that quantify this trade-off. These scores build on the information bottleneck principle, modified to not reward lossy compression. We show by example and mathematical proof that our definitions of these scores have the intuitive property of varying monotonically under cluster refinement in contrast to related proposals. Extending the information-theoretic framework beyond Shannon entropies, we furthermore derive set-matching and pair-based counterparts of the homogeneity and parsimony scores. These unify commonly used evaluation criteria and show that, in the pair-based setting, the homogeneity--parsimony trade-off recovers the receiver operating characteristic of binary classifiers. We demonstrate the framework's utility for feature selection and algorithm comparison, illustrating how considering scores jointly can clarify clustering operating points and identify Pareto-optimal solutions.
\end{abstract}

\maketitle

\section{Introduction}

Choosing among clustering approaches requires a principled evaluation strategy.
Sometimes objects come with an external partition into classes, and one wants the clustering to recover this class structure.
Assessing how closely a clustering reproduces a given class structure is the problem of \emph{external clustering validation} \citep{Rand1971ObjectiveCriteria,Strehl2002ClusterEnsembles,meila2003comparing}.
The classes serve as ``ground truth'' and may come from expert annotation or from orthogonal measurements not used to construct the clustering.
For instance, we might cluster patients using physiological measurements collected at the time of hospitalization.
We might also have follow-up data that partitions the same patients by clinical outcomes.
While many clusterings of the physiological data may be constructed, we might be primarily interested in those that recover the outcome groups.

In assessing clustering quality, two competing objectives must be balanced \citep{Rand1971ObjectiveCriteria,dom2002information,meila2003comparing,rosenberg2007v}. First, cluster assignments should predict the true labels: clusters should not mix objects from different classes. Second, clusters should be parsimonious: clusters should not split objects from the same class. These objectives are naturally in a tension that resembles the trade-off between specificity and sensitivity in binary classification. These criteria are typically considered jointly, e.g., through the receiver operating characteristic (ROC) or precision--recall curves. In contrast, in clustering the problem is often reduced to a single scalar score, such as the Rand index \citep{Rand1971ObjectiveCriteria}, the Fowlkes-Mallows index \citep{Fowlkes1983Method}, normalized mutual information \citep{Strehl2002ClusterEnsembles} and the related V-measure \citep{rosenberg2007v}.

Any scalar score sets a weighting between the competing objectives, either implicitly or explicitly, which may not suit the application. Corrections of clustering measures accounting for chance, such as the adjusted Rand index \citep{hubert1985comparing} and the adjusted mutual information \citep{vinh2009information}, partially address this issue. In these measures, the expected value under random clustering is subtracted from the raw metric. This provides an implicit penalty for clustering complexity, which adapts to the class structure and dataset size in a principled manner \citep{hubert1985comparing,vinh2009information,Romano2016AdjustingChance}. As with other scalar scores, however, equal chance-corrected scores might mask variation in homogeneity and parsimony, thereby limiting interpretability. Moreover, the correction itself depends on the assumed model of randomness \citep{wallace1983method}. We therefore argue for the value of explicitly encoding Occam’s razor in a two-objective framework, as a widely adopted analogue of ROC analysis remains lacking in external clustering validation.

To close this gap, we introduce normalized scores of clustering homogeneity and parsimony that span this trade-off.
Our proposal builds on the work of \citet{rosenberg2007v} and, like that earlier work, has conceptual roots in the information bottleneck (IB) approach, which seeks compressed representations that retain information about a target variable while minimizing retained information \citep{tishby1999information,Still2004HowMany,Slonim2005InformationbasedClustering,strouse2017deterministic}.
In external clustering validation, the target variable is the ground-truth class label and the clustering plays the role of the representation.
Framing the problem as one of optimizing two competing objectives makes it possible to distinguish \emph{Pareto-optimal} clusterings (no alternative has both higher homogeneity and higher parsimony) from preference-dependent choices \citep{Tan2022}.
In standard IB formulations, compression is an explicit goal, which means that underclustered solutions can appear Pareto-optimal despite collapsing distinct classes. The proposed parsimony score changes the objective by penalizing only structure not supported by the class labels (the class-conditional cluster entropy), so that clusters are made as simple as possible, but no simpler.

Beyond information-theoretic external validation, the literature also considers (i) set-matching metrics, which focus on the dominant label in each cluster or class \citep{ZhaoKarypis2001,Amigo2009ComparisonExtrinsic}, and (ii) pair-counting metrics, which treat clustering agreement as a binary classification problem on pairs \citep{Rand1971ObjectiveCriteria,Romano2016AdjustingChance}. We derive corresponding normalized homogeneity and parsimony scores for both settings by replacing Shannon entropies with alternative notions of uncertainty. In particular, normalized purity and inverse purity play the roles of homogeneity and parsimony for set matching, while specificity and sensitivity do so for pair-based evaluation. This perspective connects metrics that are often discussed separately.

The rest of this work is organized as follows. Section~\ref{sec_problem} defines the problem setting. Section~\ref{sec_framework} motivates the proposed evaluation framework. Section~\ref{sec_relations} derives the mathematical properties of these scores and compares them to other information-theoretic criteria on example problems. Section~\ref{sec_generalized_tradeoffs} extends homogeneity and parsimony scores to set matching and pair counting. Section~\ref{sec_applications} illustrates applications of the framework to feature selection and algorithm comparison. Section~\ref{sec_discussion} concludes with a discussion of limitations and future directions.

\section{Problem setting} \label{sec_problem}
Let $X = \{x_1, x_2, \dots, x_N\}$ be a set of $N$ objects. The ground-truth partition of $X$ into $L$ classes is denoted by
\begin{equation}
    \mathcal{C} = \{C_1, C_2, \dots, C_L\},
\end{equation}
where $C_i \, \cap \, C_j = \varnothing$ for all $i \neq j$ and $\bigcup_{i=1}^L C_i = X$.
Throughout this work, we assume non-trivial class structure where $1 < L < N$. We furthermore consider the standard setting where classes are considered to capture all meaningful structure, so any further subdivision represents unnecessary complexity.

A hard clustering of $X$ is a second partition of the same objects into $M$ clusters,
\begin{equation}
    \mathcal{K} = \{K_1, K_2, \dots, K_M\},
\end{equation}
where $M$ is in general unequal to $L$, $K_i \, \cap \, K_j = \varnothing$ for all $i \neq j$ and $\bigcup_{i=1}^M K_i = X$.
The external clustering validation problem is concerned with evaluating the closeness of the clustering $\mathcal{K}$ to the ground-truth classes $\mathcal{C}$.

This article focuses on information-theoretic solutions to this problem. For convenience, we represent partitions as labeling functions, i.e., $C : X \rightarrow \{1, 2, \dots, L\}$, where $C(x_i) = c$ for $x_i \in C_c$, and $K : X \rightarrow \{1, 2, \dots, M\}$, where $K(x_i) = k$ for $x_i \in K_k$.
When considering resampling with replacement of objects from $X$, the labeling functions $C$ and $K$ become random variables. Treating the labels as random variables, the joint distribution of class-cluster pairs is given by
\begin{equation}
    P(C=c, K=k) \equiv P(c, k) = \frac{n_{ck}}{N},
\end{equation}
where
\begin{equation}
    n_{ck} = \left| C_c \cap K_k \right|,
\end{equation}
is the $L \times M$ contingency table summarizing the number of objects in class $c$ that are assigned to cluster $k$.
The associated marginal distributions of class sizes, $P(C=c) \equiv P(c) = \sum_{k=1}^M P(c, k) = |C_c|/N$, and cluster sizes, $P(K=k) \equiv P(k) = \sum_{c=1}^L P(c, k) = |K_k|/N$, can be derived by summation.

\section{Clustering: A trade-off between homogeneity and parsimony} \label{sec_framework}

\subsection{Two objectives jointly determine clustering alignment with class structure}

Information theory formalizes the choice of a clustering $K$ as the simultaneous minimization of two objectives:
\begin{align}
    & \min_K\, H(C \mid K), \label{eq_minHCK} \\
    & \min_K\, H(K \mid C). \label{eq_minHKC}
\end{align}
The sum of these two objectives is the variation of information (VI) introduced by \citet{meila2003comparing} for comparing clusterings,
\begin{equation}
    \mathrm{VI}(C, K) = H(K \mid C) + H(C \mid K).
\end{equation}
VI is a metric on the space of partitions and as such symmetrically weights both objectives.  
However, external clustering validation compares a proposed clustering to the true class structure, which suggests it can be advantageous to consider asymmetric evaluation criteria \citep{wallace1983method}. The relative importance of minimizing either objective will depend on the downstream applications of the clusters. We therefore retain the two-objective formulation, before returning to possible scalarizations in the Discussion.

The first criterion (Eq.~\ref{eq_minHCK}) minimizes the conditional entropy of classes given clusters,
\begin{equation}
    H(C \mid K) = -\sum_{c=1}^{L} \sum_{k=1}^{M} P(c, k) \log P(c \mid k).
\end{equation}
$H(C \mid K)$ is minimized and is zero when each cluster contains elements from a single class. This criterion thus encourages cluster \emph{homogeneity} with respect to the reference partition.
Minimizing $H(C \mid K)$ is equivalent to maximizing the mutual information between class labels $C$ and clusters $K$,
\begin{equation}
        I(C, K) = H(C) - H(C \mid K),
\end{equation}
where
\begin{equation}
    H(C) = -\sum_{c=1}^{L} P(c) \log P(c)
\end{equation}
is the entropy of the class distribution. Eq.~\ref{eq_minHCK} can thus also be interpreted as maximizing the relevance of clusters to the given class partition.

The second criterion (Eq.~\ref{eq_minHKC}) minimizes splitting of classes across clusters as measured using the conditional entropy of clusters given classes,
\begin{equation}
H(K \mid C) = - \sum_{c=1}^{L} \sum_{k=1}^{M} P(c,k) \log P(k \mid c).
\end{equation}
This criterion encourages \emph{parsimony} by discouraging unnecessary subdivision of classes. Here and in the following, we use the term parsimony specifically to refer to avoiding cluster fragmentation, rather than to other model-complexity penalties such as those used, for instance, in the Akaike Information Criterion \citep{akaike2025akaike}.
For another way to see that Eq.~\ref{eq_minHKC} encourages parsimonious solutions, note that
\begin{equation}
    H(C, K) = H(K \mid C) + H(C).
\end{equation}
Since $H(C)$ does not depend on the clustering, minimizing $H(K \mid C)$ is equivalent to minimizing this joint entropy.
Therefore the second criterion encourages joint cluster-class assignments that can be encoded efficiently, related to the minimal description length principle \citep{rissanen1978modeling,dom2002information}. 

\subsection{Homogeneity and parsimony scores as standardized objective functions}

The main contribution of this work is to define standardized \emph{homogeneity} and \emph{parsimony} scores, which take values on the $[0, 1]$ interval. We define these scores as complements of the two conditional entropies, and normalize them by the maximum value they can take on across all possible clusterings for a given class partition $C$. The homogeneity score is identical to that introduced by Rosenberg and Hirschberg \citep{rosenberg2007v}, whereas the parsimony score is introduced here. Its relationship to the closely related completeness score of Rosenberg and Hirschberg is discussed in Section~\ref{sec_completeness}.

Homogeneity $h$ measures how well class structure is predicted by the clusters. Formally, it is defined as the complement of the normalized conditional entropy of class labels \citep{rosenberg2007v},
\begin{align}
    h(C, K) &= 1 - \frac{H(C \mid K)}{\max_{\tilde K} H(C \mid \tilde K)}, \\
      &= 1 - \frac{H(C \mid K)}{H(C)}, \label{eqdefh}
\end{align}
where we used that the cluster-conditional class entropy is maximized at $H(C)$ when all points are assigned to a single cluster.
Homogeneity is maximal ($h=1$) when all members of a cluster share the same class label, and minimal ($h=0$) when cluster assignments provide no information about the classes.

In analogy to homogeneity, we define the parsimony score as the complement of the conditional entropy of cluster labels divided by its maximum across all possible clusterings of the $N$ objects:
\begin{align}
    p(C, K) &= 1 - \frac{H(K \mid C)}{\max_{\tilde K} H(\tilde K \mid C)}, \\
      &= 1 - \frac{H(K \mid C)}{\log N - H(C)} \label{eqdefp}.
\end{align}
Conveniently, the maximum attainable value has a closed form expression, which can be computed considering the clustering where each object forms its own cluster (a clustering $\tilde{K}$ where $\tilde K(x_i) = i$). In this case, $H(\tilde{K}) = \log N$ and $H(\tilde{K}, C) = H(\tilde{K})$. This gives $H(\tilde{K} \mid C) = H(\tilde{K}, C) - H(C) = \log N - H(C)$. Note that the normalization is dependent on the total number of objects, $N$, in the dataset, reflecting the increased maximal clustering fragmentation possible in larger datasets. Parsimony is maximal ($p=1$) when all class members share an identical cluster label, and minimal ($p=0$) when the clustering fully fragments each class. 

Note that maximizing $h$ and $p$ is equivalent to minimizing Eqs.~\ref{eq_minHCK} and \ref{eq_minHKC}, because $h$ and $p$ are linear transformations of these conditional entropies with constant, negative coefficients.

\section{Properties of homogeneity, parsimony and alternative metrics} \label{sec_relations}

\subsection{Homogeneity and parsimony scores are monotonic under clustering refinement}

By our choice of definition, $h$ and $p$ are conveniently normalized to the $[0, 1]$ interval. Additionally, we demonstrate below that these scores behave in an intuitive manner when clusters are merged or split. A clustering $\mathcal{K}'$ is called a \emph{refinement} of $\mathcal{K}$ if each cluster in $\mathcal{K}'$ can be obtained by splitting clusters in $\mathcal{K}$ (equivalently, $\exists f : K = f(K')$) \citep{Meila2005ComparingClusterings}. With this definition, we can state the following theorem:

\begin{theorem}[Monotonicity under refinement]
    If $\mathcal{K}'$ is a refinement of $\mathcal{K}$, then
\begin{align}
    h(C, K') \geq h(C, K), \qquad p(C, K') \leq p(C, K).
\end{align}
\end{theorem}

\begin{proof}
To show this, we rewrite homogeneity as a normalized mutual information,
\begin{equation}
    h(C, K) = \frac{I(C, K )}{H(C)}.
\end{equation}
Since $K = f(K')$, $K$ is conditionally independent of $C$ given $K'$ and the data-processing inequality implies $I(C, K') \geq I(C, K)$ \citep{Cover2005ElementsInformation}. Hence $h$ is non-decreasing under refinement.
    
We similarly rewrite parsimony as a min-max normalized negative joint entropy,
\begin{equation}
    p(C, K) = \frac{\log N - H(C, K)}{\log N - H(C)}.
\end{equation}
The joint entropy is non-decreasing under refinement, $H(C, K') \geq H(C, K)$. Therefore $p$ is non-increasing under refinement.
\end{proof}

\begin{figure}[tbp]
    \centering
    \includegraphics[width=\columnwidth]{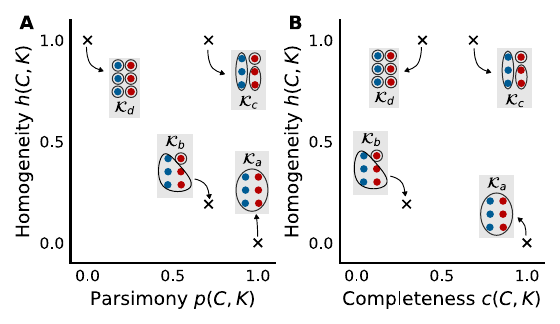}
    \caption{{\bf Comparison of parsimony and completeness scores by example.} \textbf{(A)} Parsimony--homogeneity trade-off, and \textbf{(B)} completeness--homogeneity trade-off for four clusterings $\mathcal{K}_a, \mathcal{K}_b, \mathcal{K}_c$, and $ \mathcal{K}_d$. Insets show the classes (colors) and clusterings (partitions).}
    \label{fig_completeness_example}
\end{figure}

\subsection{Comparing parsimony and completeness scores} \label{sec_completeness}

The homogeneity--parsimony scores are closely related to prior work by \citet{rosenberg2007v}. Their proposal uses the same definition of homogeneity, but pairs it with a \emph{completeness} score, defined as
\begin{equation}
    c(C, K) = 1 - \frac{H(K \mid C)}{H(K)}.
\end{equation}
Completeness, like parsimony, is maximized when each class is contained in a single cluster. However, because $H(K)$ depends on the number and balance of clusters, completeness can vary non-monotonically under refinement.

Fig.~\ref{fig_completeness_example} shows an example. Clusterings $\mathcal{K}_c$ and $\mathcal{K}_d$ are refinements of clustering $\mathcal{K}_b$ and have higher homogeneity and equal or lower parsimony, respectively. In contrast, the completeness score of the maximally fragmented clustering $\mathcal{K}_d$ is non-zero, and the completeness of clusterings $\mathcal{K}_c$ and $\mathcal{K}_d$ exceeds the value assigned to the substantially simpler clustering $\mathcal{K}_b$. The assignment of a non-zero completeness score to maximally fragmented clusterings is a general phenomenon: completeness is a cluster entropy normalized mutual information, $c(C, K) = I(C, K)/H(K)$, where for the fully fragmented clustering $\tilde K$, $I(C, \tilde K)=H(C)$ and $H(\tilde K)=\log N$. Hence $c(C,\tilde K) = H(C)/\log N$. Intuitively, as the normalization by $H(K)$ increases for fragmented clusterings, the completeness score can be overly favorable to fragmented clusterings. This problem is addressed by our definition of the parsimony score, which uses a fixed normalization independent of the clustering complexity to avoid overclustering. 

\subsection{Relation to alternative information bottleneck formulations}

\begin{figure}[tbp]
    \centering
    \includegraphics[width=\columnwidth]{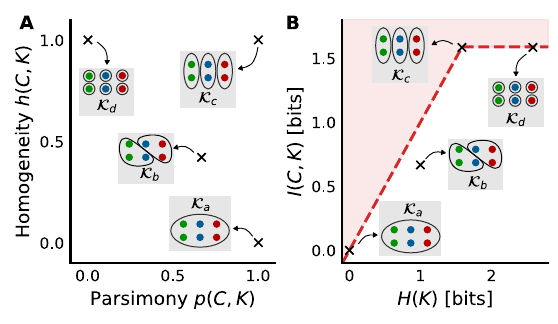}
    \caption{{\bf Comparison of homogeneity--parsimony curves to the deterministic IB trade-off.} \textbf{(A)} Parsimony--homogeneity trade-off, and \textbf{(B)} relevance--representation entropy trade-off for four clusterings $\mathcal{K}_a, \mathcal{K}_b, \mathcal{K}_c$, and $\mathcal{K}_d$. Insets show the ground-truth classes (colors) and clusterings (partitions). In panel B, the red shaded region indicates unattainable points with $I(C, K) > \min(H(K), H(C))$. Note that $H(K)$ ($= I(X, K)$ for deterministic clusterings) is to be minimized rather than maximized. }
\label{fig_parsimony_compression}
\end{figure}

As discussed, choosing a clustering represents an instance of the IB trade-off \citep{tishby1999information,Still2004HowMany,Slonim2005InformationbasedClustering}: the clustering $K$ is a compressed representation of the data $X$ that retains information about a target variable, here the ground-truth class label $C$. Eqs.~\ref{eq_minHCK},\ref{eq_minHKC} adapt the IB objectives to the external validation setting by not rewarding compression when it collapses distinct classes.

The classical IB seeks representations that maximize relevance $I(C, K)$ while compressing information about the data $I(X, K)$:
\begin{align}
    &\max_K I(C, K) \label{eqib_relevance} \\
    &\min_K I(X, K). \label{eqib_compression}
\end{align}
Maximizing relevance (Eq.~\ref{eqib_relevance}) is equivalent to maximizing homogeneity (Eq.~\ref{eq_minHCK}), as shown previously. To relate Eq.~\ref{eqib_compression} to the definition of parsimony used in clustering validation (Eq.~\ref{eq_minHKC}), we write $I(X,K)$ as
\begin{align}
    I(X, K) &= H(K) - H(K \mid X), \\
            &= H(K \mid C) + I(C, K) - H(K \mid X).
\end{align}
This decomposition reveals two terms beyond Eq.~\ref{eq_minHKC}.
As we consider deterministic hard clusterings (each $x\in X$ maps to a single cluster label), the last term is always zero \citep{strouse2017deterministic}.
The middle term, $I(C, K)$, highlights a more important difference: minimizing $H(K)$ can prefer simpler clusterings even when this simplicity is achieved by discarding relevant structure---i.e. lowering $I(C, K)$---rather than reducing $H(K \mid C)$.

Fig.~\ref{fig_parsimony_compression} shows an example. Clusterings $\mathcal{K}_a$, $\mathcal{K}_b$ have lower relevance than $\mathcal{K}_c$, but also lower $H(K)$. In the deterministic IB plane, $\mathcal{K}_a$ and $\mathcal{K}_b$ would therefore also be considered Pareto-optimal as they further compress the data. In the homogeneity--parsimony plane only clustering $\mathcal{K}_c$ is Pareto-optimal, which is precisely the clustering that perfectly recovers the classes. Replacing $H(K)$ by $H(K\mid C)$ ensures that once a clustering matches the resolution set by the reference partition (clustering $\mathcal{K}_c$), further lossy compression is not rewarded.  From the Lagrangian perspective, the parsimonious IB formulation puts a tighter bound on feasible Lagrange multipliers (see Appendix~\ref{app_pib}). Additionally, our normalized homogeneity and parsimony scores allow for an easier visualization as all values are attainable, whereas the deterministic IB plane has a non-trivial feasible region (red shaded area in Fig.~\ref{fig_parsimony_compression}B is excluded).

\section{Normalized trade-off scores for set matching and pair counting}
\label{sec_generalized_tradeoffs}

\begin{figure*}
    \centering
    \includegraphics[scale=0.8]{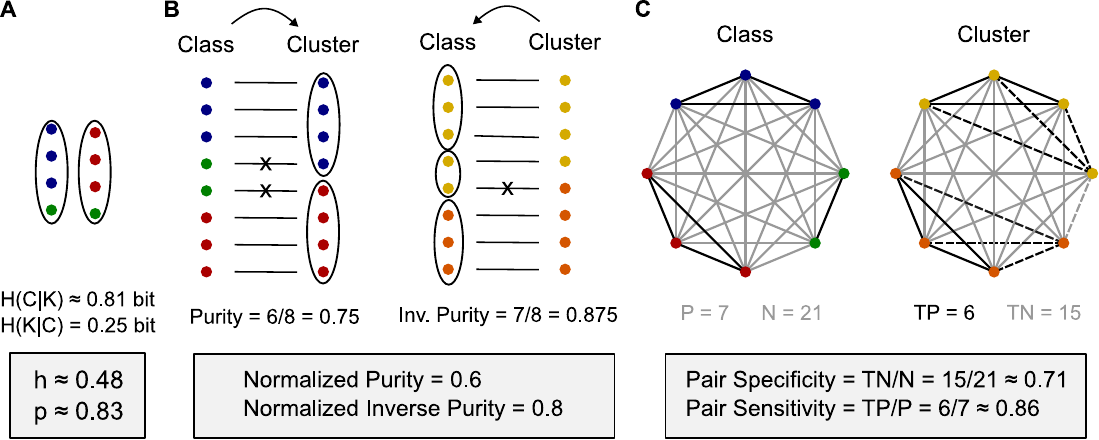}
    \caption{Illustration of three external clustering-evaluation frameworks and their corresponding normalized trade-off scores.
    \textbf{(A)} Information-theoretic (Shannon entropy), \textbf{(B)} set-matching (min entropy), and \textbf{(C)} pair-counting (collision entropy).
    In A, colors denote classes and partitions denote clusters. Conditional entropies of classes given clusters, and vice versa, yield the homogeneity $h$ and parsimony $p$ scores.
    In B, each cluster is assigned its dominant class and each class its dominant cluster, yielding normalized purity $h_\infty$ and normalized inverse purity $p_\infty$.
    In C, partitions are represented as a binary classification problem on pairs of objects, with positive labels for pairs in the same grouping (dark lines) and negative labels for pairs in different groupings (light lines). Pairwise disagreements are shown as dashed lines in the right panel, and the corresponding normalized scores are specificity $h_2$ and sensitivity $p_2$.
    }
    \label{fig_evaluation_approaches}
\end{figure*}

\subsection{Overview of alternative evaluation frameworks}

In addition to information-theoretic evaluation (Fig.~\ref{fig_evaluation_approaches}A), external clustering validation is also commonly framed in set matching or pair-classification terms. Set-matching scores, such as purity or inverse purity, assign each cluster or class to its dominant label \citep{ZhaoKarypis2001,Amigo2009ComparisonExtrinsic} (Fig.~\ref{fig_evaluation_approaches}B). Pair-counting metrics, such as the Rand Index \citep{Rand1971ObjectiveCriteria}, treat clustering-class agreement as a binary classification problem on pairs of objects (Fig.~\ref{fig_evaluation_approaches}C).

Although these evaluation strategies are usually presented separately, in the following we review how each can be derived from an information-theoretic measure. This perspective leads to analogues of the normalized homogeneity and parsimony scores for each framework.
To give intuition, recall the definition of the R\'{e}nyi entropy of order $\alpha$,
\begin{equation}
    H_{\alpha}(Z) = \frac{1}{1-\alpha}\log \sum_{z} P(z)^{\alpha},
\end{equation}
which reduces to the classical Shannon entropy in the limit $\alpha\to 1$.
As $\alpha$ increases, these generalized entropies place more weight on the dominant clusters (and classes), while smaller values of $\alpha$ emphasize the tail of the distribution. Large $\alpha$ can be useful when homogeneity and parsimony should be judged primarily by the largest groups, while smaller $\alpha$ gives a more even account of clustering alignment. Set-matching scores are connected to a conditional min-entropy construction based on the $\alpha \to \infty$ limit. Similarly, pair-counting measures are related to the $\alpha = 2$ case \citep{Romano2016AdjustingChance}, also known as collision entropy. In this way, the three frameworks emerge as related constructions motivated by generalized entropies.

Unlike Shannon entropy, generalized entropies can be defined in multiple ways, each preserving distinct algebraic or operational properties \citep{renyi1961measures,tsallis1988possible,furuichi2006information}. Rather than employing a universal definition, we adopt in each case the formulation that most directly recovers the evaluation criteria commonly used in the respective frameworks. Fig.~\ref{fig_evaluation_approaches} summarizes the resulting pair of normalized scores.

\subsection{Set matching: Purity and Inverse Purity}

To link the definitions of homogeneity and parsimony to set-matching scores, we introduce a conditional min-entropy variant based on $H_\infty$,
\begin{equation}
    H_{\infty}(C\mid K) = -\log \sum_{k} P(k) \max_{c} P(c\mid k).
\end{equation}
where we adopt an arithmetic averaging convention rather than an escort-based average \citep{amari2011geometry}. This particular conditional entropy definition is chosen because it links the negative exponential of the conditional entropies (Eqs.~\ref{eq_minHCK} and \ref{eq_minHKC}) to the so-called \emph{purity} score and its counterpart, \emph{inverse purity} \citep{ZhaoKarypis2001,Amigo2009ComparisonExtrinsic},
\begin{align}
    e^{-H_{\infty}(C\mid K)} &= \sum_{k=1}^{M} P(k) \max_{c} P(c\mid k) \equiv \mathrm{purity}(C, K), \\
    e^{-H_{\infty}(K\mid C)} &= \sum_{c=1}^{L} P(c) \max_{k} P(k\mid c) \equiv \mathrm{inv\text{-}purity}(C, K).
\end{align}
Switching from random variable to set notation recovers the commonly used definitions of purity and inverse purity
\begin{align}
    \mathrm{purity}(C, K) &= \frac{1}{N} \sum_{k=1}^M \max_c \left| C_c \cap K_k \right|, \\
    \mathrm{inv\text{-}purity}(C, K) &= \frac{1}{N} \sum_{c=1}^L \max_k \left| C_c \cap K_k \right|.
\end{align}
The minimization of this conditional min-entropy variant is thus equivalent to maximizing these scores.

In analogy to the Shannon-based derivations, we next define normalized homogeneity and parsimony scores. Purity and inverse purity are probabilities and thus already bounded between 0 and 1. However, as shown in Appendix~\ref{app_purity_extrema}, the minimum attainable values of these scores are not zero. Instead, purity is bounded below by $\max_c P(c)$, and inverse purity by $L/N$, where we recall that $L = |\mathcal{C}|$ is the support of $C$. Normalizing purity and inverse purity to the attainable range of values yields the following homogeneity and parsimony scores for set matching:
\begin{align}
    h_{\infty}(C, K) &= \frac{\mathrm{purity}(C, K) - \max_c P(c)}{1-\max_c P(c)} \\
    p_{\infty}(C, K) &= \frac{\mathrm{inv\text{-}purity}(C, K) - L/N}{1-L/N}.
\end{align}
The normalization addresses the inflation of raw purity scores in imbalanced datasets, where even a trivial clustering can score highly by predicting the majority class.
Note, though, that as with other set-matching metrics, the normalized scores depend only on the dominant label within each conditioning set. This can lead to what has been termed the problem of matching \citep{Amigo2009ComparisonExtrinsic}, in which improvements affecting minority classes are not reflected in the score. In contrast, the Shannon-based homogeneity and parsimony scores are sensitive to all labels and thus avoid this problem.

\subsection{Pair counting: false positives and false negatives}
\label{sec_paircounting}

Pair counting treats the clustering problem as a binary classification problem on pairs of objects, where the positive label is ``same class'' and the predicted label is ``same cluster''. Agreement can be summarized by two numbers: $\mathrm{FP}$, the number of pairs in the same cluster but different classes (false positives); and $\mathrm{FN}$, the number of pairs in different clusters but same classes (false negatives).
The class and cluster distribution fixes the total number of positive and negative pairs, and thus the numbers of true positive and true negative pairs. Together, these numbers determine many clustering metrics, such as the Rand index \citep{Rand1971ObjectiveCriteria}, Fowlkes-Mallows index \citep{Fowlkes1983Method}, or adjusted Rand index \citep{hubert1985comparing}. 

To relate FP and FN to the conditional entropies of Eqs.~\ref{eq_minHCK},\ref{eq_minHKC}, we make use of a different family of generalized entropies, due to Tsallis \citep{tsallis1988possible,furuichi2006information,Romano2016AdjustingChance}. In the following we sketch the derivations, and refer the reader to Appendix~\ref{app_collision_entropy} for details.

The Tsallis collision entropy of a discrete random variable $X$ is defined as \citep{havrda1967quantification,tsallis1988possible,furuichi2006information} 
\begin{equation}
    T_2(X) = 1 - \sum_{x} P(x)^2.
\end{equation}
These entropies can be estimated from sample data using an unbiased estimator due to Simpson \citep{simpson1949measurement,tiffeau2024unbiased},
\begin{equation}
    \widehat{T_2}(X) = 1 - \sum_{x} \frac{n_x(n_x-1)}{N(N-1)},
\end{equation}
where $n_x$ is the number of occurrences of $x$ in a sample of size $N$.
Using these definitions, the unbiased empirical estimates of the conditional entropies can be expressed exactly in terms of pair-counting statistics,
\begin{align}
    \binom{N}{2} \widehat{T_2}(C \mid K) &= \mathrm{FP} \\
    \binom{N}{2} \widehat{T_2}(K \mid C) &= \mathrm{FN}.
\end{align}

Pair-based homogeneity and parsimony scores can subsequently be defined using the same normalized-complement formulas as before. The normalized scores reduce exactly to the familiar metrics of specificity (TNR) and sensitivity (TPR) from binary classification:
\begin{align}
    h_2(C, K) &= 1- \frac{\mathrm{FP}}{\max_{\tilde K} \mathrm{FP}} \equiv \mathrm{TNR}, \\
    p_2(C, K) &= 1- \frac{\mathrm{FN}}{\max_{\tilde K} \mathrm{FN}} \equiv \mathrm{TPR},
\end{align}
The maximum is again taken over all possible clusterings $\tilde K$.
The number of false positives $\mathrm{FP}$ is upper-bounded by the total number of negative pairs, and the number of false negatives $\mathrm{FN}$ by the total number of positive pairs.
Therefore, the scores are equivalent to the true negative rate (TNR) and true positive rate (TPR) of the binary classifier, respectively. Homogeneity aligns with specificity (TNR) because heterogeneous clusters create false positives among the pairs predicted to be in the same cluster. Conversely, parsimony aligns with sensitivity (TPR) because splitting classes across clusters creates false negatives among the pairs predicted to be in different clusters.
In binary classification, it is more customary to plot the true positive rate (TPR) against the false positive rate (FPR), where $\mathrm{FPR} = 1-\mathrm{TNR}$. Except for this choice of axes, homogeneity--parsimony curves are direct analogues of ROC.

\begin{figure*}[t]
    \centering
    \includegraphics{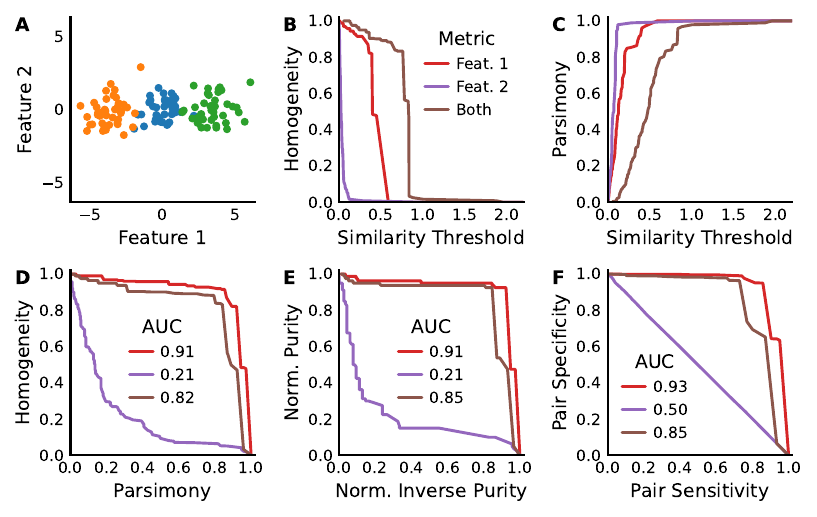}
    \caption{{\bf Trade-off curves support feature selection.} \textbf{(A)} Two labeled classes in a two-dimensional feature space, which are separated by an offset in feature~1. Single linkage agglomerative clustering is applied using either both features, feature~1 alone, or feature~2 alone. \textbf{(B)} Homogeneity $h(C, K)$ and \textbf{(C)} Parsimony $p(C, K)$ as the distance threshold $\theta$ is varied across merge points. \textbf{(D)} Homogeneity $h(C, K)$ vs. Parsimony $p(C, K)$. \textbf{(E)} Normalized purity $h_\infty(C, K)$ vs. normalized inverse purity $p_\infty(C, K)$. \textbf{(F)} Specificity $h_2(C, K)$ vs. sensitivity $p_2(C, K)$ of the induced binary classifier on pairs. In E-F the legend reports the area under curve (AUC) for each feature choice.
    }
    \label{fig_feature_selection}
\end{figure*}

\section{Applications} \label{sec_applications}

\subsection{Feature selection in agglomerative clustering}

The selection of informative features \citep{witten2010framework} or distance metrics \citep{jaskowiak2014selection} can align clusters with the desired class structure. To illustrate how our trade-off formulation can be used for graphical feature selection, we constructed a simple synthetic dataset. We drew $40$ samples from a bivariate standard normal distribution for each of three classes with class means at $(-3.5, 0)$, $(0, 0)$, and $(3.5, 0)$ (Fig.~\ref{fig_feature_selection}A). We then applied single linkage agglomerative clustering to the data. Single linkage clustering progressively merges clusters when their minimum inter-cluster distance falls below a threshold $\theta$. The problem setup implies that only the first dimension is informative for separating classes, and we thus compared the homogeneity--parsimony trade-off curves for clusterings using both features to define a distance between objects with those using either feature alone.

The results show why a direct comparison at a fixed threshold $\theta$ can be misleading (Fig.~\ref{fig_feature_selection}B, C): maximizing homogeneity at a fixed similarity threshold $\theta$ favors using both features (Fig.~\ref{fig_feature_selection}B), but would also lead to lower parsimony (Fig.~\ref{fig_feature_selection}C). For a mathematically controlled comparison, it is necessary instead to compare the homogeneity achievable with equally parsimonious clusterings. This is achieved by the homogeneity--parsimony curve traced out as $\theta$ varies (Fig.~\ref{fig_feature_selection}D). At equal parsimony, homogeneity is typically higher when relying on feature 1 alone rather than using both features. In other words, using feature 1 is Pareto optimal, as expected in this problem.

The same conclusion is reached with the set-matching (Fig.~\ref{fig_feature_selection}E) and pair-counting (Fig.~\ref{fig_feature_selection}F) formulations of the trade-off, demonstrating robustness to the choice of entropy measure. The pair-based curve is particularly simple for the uninformative feature 2, and follows the diagonal, as expected for the ROC of a random binary classifier. Such a simple baseline is not available for the two other pairs of scores. The area under the curve (AUC) of the homogeneity--parsimony trade-off, which serves as a scalar summary of performance across all thresholds, quantifies these observations. It also highlights numerical differences between the different evaluation criteria, as might be expected from the different weighting they assign to the correctness of more dominant or rare groupings. In summary, the homogeneity--parsimony trade-off provides a simple and intuitive graphical tool for feature selection.

\begin{figure}
    \centering
    \includegraphics[scale=0.9]{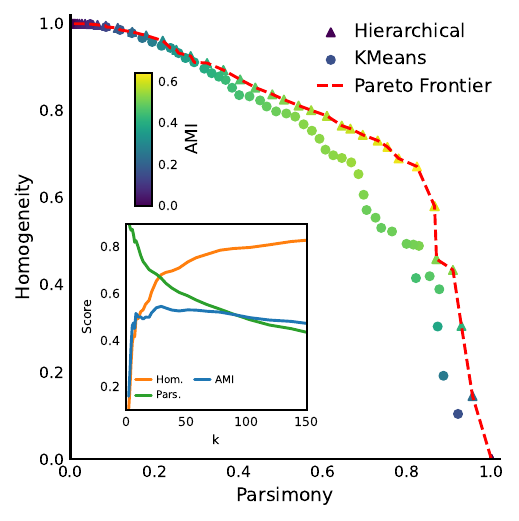}
    \caption{{\bf Pareto-optimal clustering of MNIST.} Homogeneity--parsimony curves for k-means and agglomerative clustering of the MNIST dataset.
    For k-means, the number of clusters $k$ was varied over $50$ logarithmically spaced values between $2$ and $2000$.
    For agglomerative clustering, the distance threshold for Ward's linkage criterion was varied over $50$ logarithmically spaced values between $1$ and $200$.
    Points are colored by Adjusted Mutual Information (AMI).
    The inset shows the homogeneity, parsimony, and AMI achieved by k-means clustering as a function of the number of clusters $k$.
    For computational simplicity, we randomly subsampled $2000$ images from the MNIST dataset and used the first $50$ principal components of grayscale pixel values as input to the clustering algorithms. 
    }
    \label{fig_mnist_comparison}
\end{figure}

\subsection{Algorithm comparison on MNIST}

A second illustration of how homogeneity--parsimony curves might be used is shown in Fig.~\ref{fig_mnist_comparison}. Here, we compared clustering performance on the classic MNIST dataset \citep{lecun1998gradient}, which contains small grayscale images of handwritten digits 0 to 9. To systematically obtain clusterings of varying complexity, we applied k-means \citep{steinhaus1956division} and agglomerative clustering using Ward's linkage \citep{ward1963hierarchical}. In agglomerative clustering, we varied the distance threshold for merging clusters, as in the previous subsection. In k-means, the number of clusters $k$ is controlled by a user-defined parameter. The homogeneity--parsimony analysis demonstrates that agglomerative clustering outperforms k-means on this dataset, achieving substantially higher homogeneity at intermediate parsimony levels. The figure again shows how clustering homogeneity alone can be a misleading comparison criterion: k-means can achieve higher homogeneity than agglomerative clustering depending on the parameter settings of both algorithms, but only at the cost of lower parsimony.

For comparison, we also evaluated clusterings using the Adjusted Mutual Information (AMI). The maximum AMI values are higher for agglomerative clusterings, consistent with the Pareto-optimality implied by the homogeneity--parsimony analysis. The inset of Fig.~\ref{fig_mnist_comparison} shows that AMI is comparatively insensitive near its optimum for k-means, remaining close to $\sim 0.5$ for $k \sim 8$--$100$. Over the same range, homogeneity increases markedly while parsimony decreases. This illustrates the complementary insight gained by examining the homogeneity--parsimony trade-off.

\section{Discussion} \label{sec_discussion}

External clustering validation should fundamentally be viewed as a Pareto optimization problem between homogeneity and parsimony. To support this perspective, we have defined pairs of scores for external clustering evaluation that, while closely related to existing proposals \citep{tishby1999information,meila2003comparing,rosenberg2007v}, ease assessing Pareto optimality. In particular, our homogeneity and parsimony scores behave monotonically under refinement, and the full range from 0 to 1 is attainable for clusterings of a given class distribution. We furthermore have shown how related set-matching and pair-based scores can be derived in a common information-theoretic language using generalized entropies. Taken together, these contributions might further adoption of the Pareto optimization perspective in clustering evaluation.

Establishing exact correspondences with commonly used scores in the alternative evaluation frameworks required different definitions of generalized conditional entropy and, in the pair-counting setting, an unbiased estimator of collision entropy rather than plug-in estimators was used. For pair-based metrics, the correspondence generalizes earlier results linking the VI metric to the Rand index \citep{Romano2016AdjustingChance}, while the connection between set-matching scores and min-entropy appears to be less widely known. Our derivations raise the broader question of whether alternative generalized entropy definitions \citep{renyi1961measures,tsallis1988possible,furuichi2006information} or bias-corrected entropy estimators \citep{nemenman2001entropy,grassberger2003entropy} could yield other useful clustering evaluation criteria. We leave a systematic investigation of these questions to future work.

We have restricted our attention to the validation of hard clusterings against a gold standard reference partition. Extending the framework to probabilistic or fuzzy clusterings (including overlapping cluster assignments) is an important direction for future work \citep{bezdek2013pattern,lancichinetti2009benchmarks}. Our parsimony objective is tailored to label recovery settings, where fragmentation beyond the reference classes is viewed as undesirable. In other applications, the goal may instead be to discover meaningful substructure within a class (for instance, disease subtypes within a clinically defined cohort), or to group classes with insufficient functional distinction. In these cases, other evaluation criteria might be necessary. Related discussions appear in the literature on cluster stability and model selection \citep{benhur2001stability,tibshirani2001estimating}. Finally, when external labels are unavailable, clustering selection must rely on internal validation criteria, such as the Silhouette coefficient \citep{rousseeuw1987silhouettes} or the Davies--Bouldin index \citep{davies1979cluster}.

An important conceptual choice in our approach is to define scores that have a fixed normalization for a given partition of objects into classes, independent of the specific clustering. This allows for a fairer comparison across clusterings than alternative normalizations, which also depend on the clustering. However, some care must be taken when comparing scores across datasets of different sizes. The parsimony score depends not only on the effective number of classes, but also on the total number of objects. It thus increases with the number of objects, even if the class distribution remains unchanged. This reflects that the score measures the proportion of fragmentation relative to the maximally fragmented clustering, which inherently increases with dataset size.

While this paper argues for the benefits of a two-objective formulation, it has so far provided little practical guidance on how to choose a single clustering from the Pareto frontier. In Appendix~\ref{app_scalarizations}, we introduce the Q-measure as a weighted harmonic mean of homogeneity and parsimony to provide a scalar clustering quality measure. This measure has some advantages over the related V-measure \citep{rosenberg2007v}, but the problem of choosing the weight parameter remains. Ultimately, the specific application dictates what level of homogeneity is required, and how much parsimony can be sacrificed to achieve it. The contributions of the two-objective framework are therefore twofold: it uses Pareto optimality to prune candidate clusterings in a principled manner, and it offers the trade-off curve as a visual tool to support the final, application-dependent decision.

An important open problem is to compare the parsimony criterion systematically with other complexity penalties. In Section~\ref{sec_applications}, we used a simple example problem to compare the two-objective framework with adjusted mutual information, which down-weights the relevance of a clustering by subtracting the mutual information expected under a hypergeometric model of randomness \citep{vinh2009information,Romano2016AdjustingChance}. This expected value increases with the number of clusters and thus provides an implicit parsimony correction. Reduced mutual information, by contrast, penalizes non-parsimonious solutions by adding a term proportional to the description length of the contingency table \citep{newman2020improved}. This adjustment, as well as an earlier precursor by Dom \citep{dom2002information}, is also conceptually related to the parsimony criterion, since contingency tables for simpler clusterings can be expressed more succinctly. Developing a mathematical theory that unifies these complexity penalties would be a worthwhile goal for future work.

For simplicity of exposition we have relied on simple toy problems to accompany the theoretical derivations. In ongoing research, we have applied this framework to benchmark clustering algorithms for immunological data, extending our prior work in this area \citep{turner2026evolution}. This will be the subject of a separate publication. More generally, we hope that the homogeneity--parsimony trade-off will be a broadly useful addition to the clustering evaluation toolbox across the myriad real-world applications of cluster analysis. To aid adoption ``in the wild'', we have made a reference implementation for computing these scores available as an open-source Python package \citep{github_repo}.

\vspace{5pt}

\textbf{Acknowledgements.} The author thanks Rishika Saxena for uncovering the unexpected behavior of the completeness score under refinement in a clustering application setting, Mahdad Noursadeghi for tirelessly advocating for a ROC-style clustering evaluation, and James Henderson for discussion of non-Shannon entropy measures. Claude Sonnet 4.6 (Anthropic) and ChatGPT 5.4 (OpenAI) were used for coding assistance and copyediting with all outputs fully reviewed by the author. Feedback on the draft manuscript from Reemon Spector, Chris Watkins, Lucia Guillamet Garcia, and other members of the Q-Immuno lab is gratefully acknowledged, as well as financial support of the Wellcome Trust (306550/Z/23/Z). 
\bibliographystyle{apsrev4-2}
\bibliography{references.bib}

\appendix

\section{The parsimonious information bottleneck Lagrangian} \label{app_pib}

In the following, we compare the Lagrangian formulation of the Parsimonious Information Bottleneck (PIB), which penalizes $H(K \mid C)$, with that of the Deterministic IB (DIB).
Recall that any multi-objective optimization problem can be recast as a scalar optimization problem with constraints on the remaining objectives. In external clustering validation, we can minimize fragmentation $H(K \mid C)$ among all solutions with constant inhomogeneity $H(C \mid K)$. This leads to the parsimonious information bottleneck (PIB) Lagrangian,
\begin{equation}
    \mathcal{L}_{\mathrm{PIB}}(\beta) = H(K \mid C) + \beta_{\mathrm{PIB}} H(C \mid K),
\end{equation}
where $\beta_{\mathrm{PIB}} \geq 0$ is a Lagrange multiplier.
Compared to the DIB Lagrangian \citep{strouse2017deterministic},
\begin{equation}
    \mathcal{L}_{\mathrm{DIB}} = H(K) - \beta_{\mathrm{DIB}} I(C, K),
\end{equation}
the PIB Lagrangian puts a tighter bound on feasible values of $\beta$ as
\begin{equation}
    \beta_{\mathrm{DIB}} = \beta_{\mathrm{PIB}} + 1.
\end{equation}
This relationship suggests that the DIB criterion might favor underclustering when $\beta_{\mathrm{DIB}} \leq 1$.
We note that at $\beta_{\mathrm{PIB}} = 1$ the PIB Lagrangian equals the $\mathrm{VI}$ measure \citep{meila2003comparing}. The PIB Lagrangian thus represents a one-parameter generalization of the variation of information that allows unequal weights.

\section{Extrema of purity and inverse purity} \label{app_purity_extrema}

Both purity and inverse purity attain maximal values of $1$ for the perfect clustering in which each class is contained in a single cluster.

\begin{lemma}
The minimal attainable value of $\mathrm{purity}(C, K)$ for any clustering is $\max_c P(c)$.
\end{lemma}

Let $c^\star$ be the index of the largest class. The baseline probability of drawing from this class is
\begin{equation}
    P(c^\star) = \max_c P(c).
\end{equation}
For any cluster $k$, the maximum conditional probability is bounded by the conditional probability of the largest class,
\begin{equation}
    \max_c P(c\mid k) \geq P(c^\star \mid k)
\end{equation}
Summing over all clusters $k$ gives
\begin{equation}
    \sum_k P(k)\max_c P(c\mid k) \geq \sum_k P(k) P(c^\star \mid k).
\end{equation}
By the law of total probability, the right-hand side equals $P(c^\star)$, giving
\begin{equation}
    \mathrm{purity}(C, K) \geq P(c^\star) = \max_c P(c).
\end{equation}
This minimum is attained by the single-cluster partition.

\begin{lemma}
The minimal attainable value of $\mathrm{inv\text{-}purity}(C, K)$ is $|\mathcal{C}|/N$.
\end{lemma}

To show this, we will make use of the fact that inverse purity is monotonic under refinement of the clustering. In the present finite-sample hard-partition setting, the minimal value is therefore attained by the fully fragmented clustering in which each object forms its own cluster. For such a clustering, for each class $c$, the largest cluster-conditional probability is $1/n_c$, giving
\begin{equation}
    \mathrm{inv\text{-}purity}(C, K)
    = \sum_c \frac{n_c}{N}\frac{1}{n_c}
    = \frac{|\mathcal C|}{N}.
\end{equation}

\section{Relation of pair-counting and collision entropies} \label{app_collision_entropy}

\subsection{Tsallis entropy definition}

The Tsallis $q$-entropy \citep{havrda1967quantification,tsallis1988possible,furuichi2006information} is defined as
\begin{equation}
    T_{\alpha}(Z) = \frac{1}{\alpha-1}\left(1 - \sum_{z} P(z)^\alpha\right).
\end{equation}
It is a monotone transformation of the R\'{e}nyi entropy
\begin{equation}
    T_\alpha(Z) = \frac{1-\exp\bigl((1-\alpha)H_\alpha(Z)\bigr)}{\alpha-1},
\end{equation}
and as such an alternative one-parameter generalization of the Shannon entropy.
At order $\alpha=2$, the Tsallis entropy is equal to
\begin{equation}
    T_2(Z) = 1-\sum_{z} P(z)^2,
\end{equation}
which is also called \emph{Gini-Simpson diversity index}.
Following Furuichi \citep{furuichi2006information}, we define conditional Tsallis entropies using 
\begin{equation}
    T_{\alpha}(X \mid Y) = \sum_y p(y)^{\alpha} T_{\alpha}(X\mid y),
\end{equation}
to maintain the chain rule 
\begin{equation}
    T_{\alpha}(C\mid K) = T_{\alpha}(C, K) - T_{\alpha}(K).
\end{equation}

\subsection{Coincidence probabilities determine pair counts}

The Tsallis collision entropy is the complement of the coincidence probability $\kappa(Z)$,
\begin{equation}
    \kappa(Z) = \sum_{z} P(z)^2 = 1-T_2(Z).
\end{equation}
This quantity measures the probability that two independent draws from the distribution are identical. $\kappa$ is also known as Simpson's index in ecology \citep{simpson1949measurement}. The associated R\'{e}nyi entropy of order $2$ (also called \emph{collision entropy}) is related to $\kappa$ via $H_2(Z) = -\log \kappa(Z)$.

If we have a finite sample from $Z$ of size $N$ in which the $i$-th element has been observed $n_i$ times, then we can estimate $\kappa$ without bias using the following estimator \citep{simpson1949measurement}:
\begin{equation}
    \hat{\kappa} = \sum_i \frac{n_i (n_i-1)}{N (N-1)} = \sum_i \binom{n_i}{2} \Big/ \binom{N}{2}.
\end{equation}
For a finite sample, the pair counts can be written in terms of the contingency table $n_{ck}$ and marginals $n_c$, $n_k$, which in turn can be written in terms of $\hat \kappa$,
\begin{align}
    \mathrm{TP} &= \sum_{c,k} \binom{n_{ck}}{2} = \binom{N}{2} \, \hat{\kappa}(C,K), \\
    \mathrm{FP} &= \sum_{k} \binom{n_{k}}{2} - \mathrm{TP} = \binom{N}{2} \, \bigl(\hat{\kappa}(K) - \hat{\kappa}(C,K)\bigr), \\
    \mathrm{FN} &= \sum_{c} \binom{n_{c}}{2} - \mathrm{TP} = \binom{N}{2} \, \bigl(\hat{\kappa}(C) - \hat{\kappa}(C,K)\bigr)\\
    \mathrm{TN} &= \binom{N}{2} - \mathrm{TP} - \mathrm{FP} - \mathrm{FN}.
\end{align}

\subsection{Normalization and equivalence to TNR/TPR}

The maximal number of false positive pairs is attained when all objects are assigned to a single cluster. In that case every negative pair is predicted positive, hence
\begin{equation}
    \max_{\tilde K}\mathrm{FP}=\mathrm{FP}+\mathrm{TN}.
\end{equation}
Similarly, the maximal number of false negative pairs is attained when each object forms its own singleton cluster. In that case every positive pair is predicted negative, hence
\begin{equation}
    \max_{\tilde K}\mathrm{FN}=\mathrm{FN}+\mathrm{TP}.
\end{equation}
Therefore the normalized complement scores are
\begin{align}
    h_2(C, K)
    &=1-\frac{\mathrm{FP}}{\max_{\tilde K}\mathrm{FP}}
      =1-\frac{\mathrm{FP}}{\mathrm{FP}+\mathrm{TN}}\\
    &  =\frac{\mathrm{TN}}{\mathrm{TN}+\mathrm{FP}}
      =\mathrm{TNR}, \\
    p_2(C, K)
    &=1-\frac{\mathrm{FN}}{\max_{\tilde K}\mathrm{FN}}
      =1-\frac{\mathrm{FN}}{\mathrm{FN}+\mathrm{TP}}\\
    &  =\frac{\mathrm{TP}}{\mathrm{TP}+\mathrm{FN}}
      =\mathrm{TPR}.
\end{align}
Thus, in the pair-counting formulation, homogeneity and parsimony are mathematically equivalent to the empirical unbiased estimators of the true negative rate (TNR, specificity) and true positive rate (TPR, sensitivity) of the induced binary classification problem on pairs. 

\section{The Q-measure as a scalar clustering quality metric} \label{app_scalarizations}

While we emphasize the advantages of two-objective formulations of external clustering validation, our homogeneity and parsimony scores can also be combined into a single scalar score for applications requiring a single metric. We define the Q-measure of clustering quality as a weighted harmonic mean of homogeneity and parsimony:
\begin{equation}
    Q_\beta(C, K) = \frac{(1 + \beta) h p}{\beta h + p},
\end{equation}
where $\beta$ controls the relative importance of homogeneity and parsimony, $\lim_{\beta \to 0} Q_\beta = h$ and $\lim_{\beta \to \infty} Q_\beta = p$.
Our definition of the Q-measure is inspired by the V-measure \citep{rosenberg2007v}, a widely used scalar score based on homogeneity and completeness, $V_\beta = [(1 + \beta) h c]/[\beta h + c]$, but it replaces completeness with parsimony.

Table~\ref{tab:v-q-measures} compares how Q-measure and V-measure rank the clusterings from Figure~\ref{fig_completeness_example} at $\beta=1$. Both measures agree on ranking clustering $\mathcal{K}_c$ best, while clustering $\mathcal{K}_a$, which puts all objects into a single cluster, is ranked worst. However, Q-measure and V-measure differ in their ranking of clusterings $\mathcal{K}_d$ and $\mathcal{K}_b$. V-measure ranks clustering $\mathcal{K}_d$, which puts each point in its own cluster, over clustering $\mathcal{K}_b$. Q-measure ranks $\mathcal{K}_b$ higher and assigns a score of zero to $\mathcal{K}_d$. This better matches intuition, since clustering $\mathcal{K}_d$ is the least parsimonious. Q-measure can thus serve as an alternative scalarization to V-measure, which penalizes overclustering more strongly.
The harmonic mean provides a natural criterion for problems where overall solution quality is bottlenecked by poor performance on any of the competing criteria.
In practice, optimizing $Q_\beta$ over a range of $\beta$ yields solutions along the Pareto frontier of the homogeneity--parsimony trade-off, from which one can select a clustering suited to the specific application.

\begin{table}[hbt]
    \centering
    \caption{Comparison of the Q-measure and V-measure for the clusterings ($\mathcal{K}_a$--$\mathcal{K}_d$) shown in Fig.~\ref{fig_completeness_example} (at $\beta=1$ for both scores).}
    \label{tab:v-q-measures}
    \begin{ruledtabular}
    \begin{tabular}{lrr}
        Clustering & Q-measure & V-measure \\
        \midrule
        $\mathcal{K}_a$ & 0.00 & 0.00 \\
        $\mathcal{K}_b$ & 0.30 & 0.23 \\
        $\mathcal{K}_c$ & 0.83 & 0.81 \\
        $\mathcal{K}_d$ & 0.00 & 0.56 \\
    \end{tabular}
    \end{ruledtabular}
\end{table}

\end{document}